\documentclass[a4paper,11pt]{article}

\usepackage[margin=1.2in]{geometry}
\usepackage{latex-macros}
\usepackage{subcaption}

\usepackage[dvipsnames]{xcolor}
\definecolor{lblue}{HTML}{908cc0}
\definecolor{mblue}{HTML}{519cc8}
\definecolor{hblue}{HTML}{1d5996}
\definecolor{lred}{HTML}{cb5501}
\definecolor{mred}{HTML}{f1885b}
\definecolor{hred}{HTML}{b3001e}
\definecolor{darkorange}{HTML}{7a2000}

\usepackage{pst-eucl}
\usepackage{multirow}

\usepackage{amsmath, amssymb, amsthm}
\usepackage{bm}
\usepackage{cancel}

\newtheorem{defn}{Definition}
\newtheorem{prop}{Proposition}
\newtheorem{corl}{Corollary}
\newtheorem{thrm}{Theorem}
\newtheorem{lem}{Lemma}
\theoremstyle{plain}

\usepackage[ruled]{algorithm}
\usepackage{algpseudocode}
\usepackage{caption}

\usepackage{tikz}
\usepackage{pgfplots}
\usepackage{tkz-graph}
\usetikzlibrary{
    matrix,chains,positioning,decorations.pathreplacing,arrows,patterns,
    svg.path, shapes.geometric}
\pgfplotsset{compat=newest, clip bounding box=default tikz}
\usepgfplotslibrary{fillbetween}

\usepackage[capitalize]{cleveref}

\newcommand{\normal}{\mathsf{N}}
\newcommand{\half}{\scalebox{0.75}}
\newcommand{\vtx}{\mathrm{vert}}
\newcommand{\R}{\mathbb{R}}

\newcommand{\Kb}{\mathbb{K}}

\newcommand{\Nc}{\mathcal{N}}

\allowdisplaybreaks

\usepackage[square,numbers,sort]{natbib}

\crefname{defn}{definition}{definitions}
\crefname{lemma}{lemma}{lemmas}

\begin{document}
\title{\textsc{
    A Tropical Approach to Neural Networks
    with Piecewise Linear Activations}}
\author{
    \textsc{Vasileios Charisopoulos}\thanks{\quad
        School of Operations Research and Information Engineering, Cornell University,
        Ithaca, NY 14850, USA.
   \email{vc333@cornell.edu}}
   \qquad
   \textsc{Petros Maragos}\thanks{\quad School of Electrical \& Computer
   Engineering,
       National Technical University of Athens, Zografou Campus, 15773 Athens,
       Greece.
       \email{maragos@cs.ntua.gr}}}
\date{May 22, 2018; Revised \today}
\maketitle
\begin{abstract}
We present a new, unifying approach following some recent developments on the
complexity of neural networks with piecewise linear activations. We treat neural
network layers with piecewise linear activations as \textit{tropical
polynomials}, which generalize polynomials in the so-called $(\max, +)$ or
\textit{tropical} algebra, with possibly real-valued exponents.
Motivated by the discussion in~\cite{MPC+14}, this approach enables us to
refine their upper bounds on linear regions of layers with ReLU or leaky ReLU
activations to $\min\left\{ 2^m, \sum_{j=0}^n \binom{m}{j} \right\}$, where $n,
m$ are the number of inputs and outputs, respectively. Additionally, we recover
their upper bounds on maxout layers. Our work follows
a novel path, exclusively under the lens of tropical geometry, which is
independent of the improvements reported in~\cite{Aro+16,SerTjaRam18}.
Finally, we present a geometric approach for effective counting of linear
regions using random
sampling in order to avoid the computational overhead of exact counting
approaches.
\end{abstract}

\section{Introduction}
In the past decade, multilayered architectures of neural networks have
enjoyed an unprecedented growth in popularity, with the introduction of the
paradigm of \textit{deep learning}~\cite{Bengi09,KriSuHi12,GoodBenCour16}.
Deep neural networks consist of the composition of many
layers of neurons, which are typically fed through nonlinear activation functions. Two of the most widely used such activations are rectifier linear units (ReLUs) and \textit{maxout} units, which are both piecewise-linear. ReLUs have been shown to outperform traditional choices of activation functions in empirical studies~\cite{GloBorBen11, MHN13}, while maxout networks~\cite{GWM+13} were also
quickly adopted after their introduction~(see e.g. \cite{ZTP+14}), as they were empirically validated to integrate well with an averaging technique called \textit{dropout}~\cite{SNK+14}.
The output of a neural network employing either of the above
activations is a piecewise-linear function; \cite{PMB13,MPC+14} argued that the number of \textit{linear regions} (i.e. regions of the input space where the output is locally linear) designated by a neural network is a good indicator of its expressive power, and consequently sought to derive upper bounds.

We briefly sketch the outline of this paper:
\begin{enumerate}
    \item We show that families of piecewise-linear activation functions
    employed in (deep) neural networks naturally correspond to so-called
    \textit{max-polynomials} or \textit{tropical polynomials} with real
    exponents.
    We obtain bounds on the number of linear regions of piecewise-linear neural network layers employing a certain duality between tropical polynomials and
    their corresponding Newton Polytopes.
    \item We identify an efficient way for counting linear regions of neural network layers in practice, which adapts a randomized algorithm for counting extreme points of convex polytopes to the Minkowski sum setting.
\end{enumerate}

\subsection{Notation and terminology}
For the reader's convenience, it is necessary to explain the notation and
terminology used in subsequent chapters, as well as a few conventions that we
will follow. We denote by $\R$ the line of real numbers and use $\R_{\max}$ for the extended real numbers $\R_{\max} := \R \cup \{ -\infty \}$.
We denote scalars by regular lowercase font, such as $x \in \R$;
vectors by bold lowercase, such as $\bm{x} \in \R^n$; and matrices by bold uppercase, such as $\bm{X} \in \R^{m \times n}$. We follow the convention of
column vectors, unless explicitly stated otherwise. We denote the set of
indices $[n] := \set{1, \dots, n}$, and write $\norm{\cdot}$ for the
$\ell_2$ norm, \( \norm{\bm{x}} := \left( \sum_{i=1}^n \abs{x_i}^2 \right)^{1/2}. \)

We also follow the lattice-theoretic notation of the mathematical morphology
community with regard to the idempotent operators $\max, \min$, in the spirit
of~\cite{Mar17}. Specifically, given $v_i \in \R$:
\begin{align}
	\bigvee_{i=1}^n v_i := \max \left( v_1, \dots v_n \right), \quad
    \bigwedge_{i=1}^n v_i := \min \left( v_1, \dots v_n \right)
\end{align}
Finally, we write $\normal(\bm{0}, \bm{I}_d)$ for the multivariate centered normal vector with unit covariance matrix.

\subsection{Related Work}
The use of tropical geometry to bound the representation power and complexity
of learning models has been pioneered by~\cite{PaSt04} in their seminal paper, which used tropical geometry to assess the effect of graphical model parameters on the solutions of the corresponding inference problems. This line of work was later extended in more general settings,
ranging from applications on computational biology~\cite{PaSt05} to the
identifiability of the Restricted Boltzmann Machine~\cite{CueMorStu10}.

Bounds on the inference regions of neural networks were, to the best of our
knowledge, first given in~\cite{MSE89}, who considered a $2$-layer neural
network with $0$-$1$ activations. More than two decades later,
in~\cite{PMB13,MPC+14}, the authors rederived essentially the same bounds for
layers of neural networks with convex piecewise linear activations, which are
more common in contemporary architectures. These bounds were also employed
in~\cite{RagPooKle+16}, where the authors are concerned with identifying
varying measures of expressivity of deep neural networks.
Other authors~\cite{SerTjaRam18,Aro+16} have since refined these types
bounds and proposed practical ways of counting linear regions of neural
networks~\cite{SerTjaRam18,SerRam18}.
Concurrently to the publication of the first edition of this
paper,~\cite{ZhaNaiLim18} established a similar correspondence between inference
regions of neural networks and tropical geometry. However, to the best of our
knowledge, such a connection had already been encountered in~\cite{ChMa17},
where it was observed that maxout and ReLU activations are essentially
represented by their corresponding Newton polytopes. Finally,
in~\cite{CalGauPos18} the authors design universal approximators of certain
classes of data using an argument related to the \textit{Maslov
dequantization}, an important transform in tropical algebra.

\section{Background}
\subsection{The tropical semiring}
The term ``tropical semiring'' refers to one of the $(\max, +)$ or $(\min, +)$
semirings, which are the algebraic structures defined as $(\R_{\max}, \max, +)$
and $(\R_{\min}, \min, +)$, respectively. In short, ordinary ``addition'' is
replaced by the maximum or minimum, and ``multiplication'' is replaced by
ordinary addition. We use the symbols $\vee, \boxplus$ to refer to matrix/vector
addition and multiplication in the case of the $(\max, +)$ semiring; a notable
exception is when the operands are scalars, where we may use just $\max$/$\min$
and $+$ for simplicity. Table~\ref{tab:linear_maxalgebra_correspondences} summarizes some important correspondences between linear and $(\max, +)$
algebra. Vector operations generalize in the obvious way: for example, the dot product is as follows:
\begin{align}
    \bm{c}^\top \boxplus \bm{d} & :=
        \bigvee_{i=1}^k c_i + d_i
        \label{eq:maxplus_dotproduct}
\end{align}
Similar definitions hold for the $(\min, +)$ semiring.

\begin{table}
    \centering
    \caption{Correspondences between linear and $(\max, +)$ arithmetic}
    \begin{tabular}{c | c}
        Linear arithmetic & $(\max, +)$ arithmetic \\ \hline
        + & $\max$ \\
        $\times$ & + \\
         0  & $-\infty$ \\
         1  & 0 \\
         $x^{-1} = 1/x$ &
         $x^{-1} = -x$ \\ \hline
    \end{tabular}
    \label{tab:linear_maxalgebra_correspondences}
\end{table}

\subsection{Elements of Discrete \& Tropical Geometry}
Subsequent sections make extensive use of results \& definitions from discrete
geometry, which we briefly present here; we mainly follow~\cite{Zieg95}.
First, we need the notion of a convex hull:
\begin{defn} \label{defn:convex_conic_hull}
    Let $\bm{v}_1, \dots, \bm{v}_m$ be a collection of points in $\R^n$.
    Their \textbf{convex hull} is defined as
    \begin{equation}
    \mathrm{conv}\{ \bm{v}_i : i \in [m] \} :=
    \sum_{i=1}^m \lambda_i \bm{v}_i, \quad \lambda_i \geq 0, \ \sum_{i=1}^m
    \lambda_i = 1.
    \end{equation}
\end{defn}
A (convex) \textit{polytope} $P \subseteq \Rbb^n$ is a set which can be written
as the convex hull of a finite set of points; if these points are
known, we say that $P$ admits a $\cV$-representation:
\begin{equation}
	P = \mathrm{conv}\set{\bm{v}_1, \dots, \bm{v}_k} \label{eq:V-repr}
\end{equation}
Additionally, we write
\begin{equation}
	\mathrm{vert}(P) := \set{\bm{v} \mmid \bm{v} \text{ is a vertex of } P}.
	\label{eq:vertex-set}
\end{equation}
We define the \textbf{upper hull} $P^{\max}$ of a polytope $P$ as
\begin{equation}
	P^{\max} := \left\{
		(\lambda, \bm{x}) \in P: (t, \bm{x}) \in P \Rightarrow t \leq
		\lambda \right\}.
	\label{eq:upper-hull}
\end{equation}
The \textbf{lower hull}, $P^{\min}$, is defined in an analogous fashion.
We also deal with \textit{Minkowski sums} of convex polytopes, which are
defined as follows:
\begin{defn}\label{defn:minkowski_sum}
    Let $P, Q \in \R^n$ be convex polytopes. Their \textbf{Minkowski sum}
    is
    \begin{align}
    P \oplus Q & := \{ \bm{p} + \bm{q} \in \R^n: \bm{p} \in P, \ \bm{q} \in Q
    \} \\
    	&= \mathrm{conv}\set{\bm{p} + \bm{q} \mmid
    			\bm{p} \in \vtx(P), \; \bm{q} \in \vtx(Q)}, \nonumber \\
    \end{align}
\end{defn}
where we can write the latter if their $\cV$-representations are given.
Obviously, the Minkowski sum of two or more convex polytopes is also a convex polytope. Another fundamental object we employ is the \textbf{normal cone} to a point of a polytope:
\begin{defn} \label{defn:normal-cone}
	The \textbf{normal cone} to a polytope $P$ at $\bm{x}$ is
	\begin{equation}
		N_{P}(\bm{x}) := \set{\bm{c} \in \Rbb^n \mmid
		\bm{c}^\top (\bm{z} - \bm{x}) \leq 0, \; \forall \bm{z} \in P}.
	\end{equation}
\end{defn}
Lemma~\ref{lemma:normal-cone-span} tells us that the normal cones of a polytope
cover the whole underlying space:
\begin{lem} \label{lemma:normal-cone-span}
	Let $P \subset \Rbb^n$ be a polytope, and denote $\mathrm{vert}(P)$ for its
	collection of vertices. Then
	\(  \displaystyle
		\setU_{\bm{v} \in \mathrm{vert}(P)}
        N_P(\bm{v}) = \Rbb^n.
	\)
\end{lem}
\begin{proof}
    Consider an \textbf{arbitrary} vector $\bm{c} \in \Rbb^n$ and its
    associated linear functional $\bm{x} \mapsto \bm{c}^\top \bm{x}$, which
    attains a maximizer on $P$.	By the fundamental theorem of linear
    programming \cite{Vand14}, all
    linear functionals attain their maxima / minima on one of the vertices of $P$, which means that $\exists \bm{v} \in \mathrm{vert}(P)$ such that
    \[
        \bm{c}^\top \bm{v} \geq \bm{c}^\top \bm{x}, \; \forall \bm{x} \in P
        \Rightarrow \bm{c} \in N_P(\bm{v}).
    \]
\end{proof}
Given a cone, its \textbf{solid angle} is as follows:
\begin{defn} \label{defn:solid-angle}
    Consider a convex cone $K \subseteq \Rbb^n$. The \textbf{solid angle} of $K$,
    $\omega(K)$, is defined as
    \begin{align*}
        \omega(K) & := \int_{K} \exp\left(-\pi
        \norm{\bm{x}}^2 \right) \dd{\bm{x}} \\
        &=
        \frac{1}{(2 \pi)^{n/2}}\int_{K}
        \exp\left(-\frac{\norm{\bm{x}}^2}{2}\right) \dd{\bm{x}}
    \end{align*}
\end{defn}
Note that the latter expression in~\Cref{defn:solid-angle} is equal to
$\prob{\bm{g} \in K}, \; \bm{g} \sim \normal(\bm{0}, \bm{I}_n)$, implying the
following:
\begin{corl}
    Given a convex polytope $P$, the solid angles of the normal cones to its
    vertices form a probability distribution, i.e.
    \(  \displaystyle
        \sum_{\bm{v} \in \mathrm{vert}(P)} \omega(N_P(\bm{v})) = 1.
    \)
\end{corl}
\begin{proof}
    Obviously, $\omega(N_P(\bm{v})) \geq 0, \; \forall \bm{v}$.
    Using~\Cref{defn:solid-angle}, we may write
    \begin{align*}
        \sum_{\bm{v} \in \mathrm{vert}(P)} \omega(N_P(\bm{v})) &=
        \sum_{\bm{v} \in \mathrm{vert}(P)} \prob{g \in N_P(\bm{v})} \\
        &=
        \prob{\setU_{\bm{v} \in \mathrm{vert}(P)} \set{g \in N_P(\bm{v})}} = 1,
    \end{align*}
    where we made use of the fact that $\omega(N_P(\bm{v}_i) \cap
    N_P(\bm{v}_j)) = 0$ and Lemma~\ref{lemma:normal-cone-span}.
\end{proof}
Finally, we call a set of $m$ points in $\Rbb^n$ to be in \textbf{general
position} if no $n + 1$ of them lie on a common hyperplane.

\subsubsection{Tropical Polynomials}
We briefly introduce tropical polynomials, on which we heavily rely in our
approach. A polynomial in $n$ variables with coefficients from a field $\Kb$,
$p \in \Kb[x_1, x_2, \dots x_n]$, is defined as
\[
	p(\bm{x}) = \sum_{i} a_i \cdot \bm{x}^{\bm{u}^i}, \quad
    \bm{u}^i \in \mathbb{N}^n
\]
so that the exponent $\bm{u}^i$ results in $\bm{x}^{\bm{u}^i} = x_1^{u^i_1}
x_2^{u^i_2} \cdots x_n^{u_n^i}$. If one relaxes the assumption on the exponent
$\bm{u}^i$ being an integer vector, and allowing $\bm{u}^i \in \mathbb{R}^n$
instead, we then call the resulting expression a
\textbf{signomial}~\cite{DufPet67}. Signomials and their positive-coefficient
special cases, called \textit{posynomials}, appear in the context of geometric
programming.
In tropical geometry, polynomials exhibit fundamental differences due to the
underlying binary operations. The multi-exponent $\bm{u}^i$ is
replaced by a vector of coefficients  $\bm{c}_i$, and exponentiation becomes
the dot product. A tropical polynomial can be viewed as the ``tropicalization''
of an ordinary polynomial over a non-Archimedean field. For
further details, we refer the reader to~\cite{MaSt15}. However, given
that we wish to model activations of neural networks which can have real
coefficients, we adopt the corresponding terminology and talk about
\textbf{tropical signomials} (also referred to as \textit{maxpolynomials}
in~\cite{But10}), where $\bm{c}_i \in \mathbb{R}^n$ as shown
below:
\begin{equation}
	h(\bm{x}) = \bigvee_{i=1}^m b_i + \bm{c}_i^\top \bm{x}, \quad
    \bm{c}_i \in \mathbb{R}^n
    \label{eq:tropical_signomial}
\end{equation}
In the sequel, we will use the terms ``polynomials'' and ``signomials'' interchangeably, i.e. tropical polynomials will always allow real exponents.
We say a polynomial is of \textit{rank} $k$ if it is the maximum of $k$ terms.

A \textit{hypersurface} associated with a ``classical'' polynomial $p$ is defined as its zero set, $V(p) = \{ \bm{x} \in \R^n: p(\bm{x}) = 0 \}$.
On the contrary, the ``zero locus'' of a tropical polynomial $p$ is simply
the set of points where the maximum is attained by more than one of its terms:
\begin{equation}
	V(p) = \{ \bm{x} \in \R^n_{\max}: p(\bm{x}) \mbox{ is singular }  \}
    \label{eq:tropical_hypersurface}
\end{equation}
An example of a tropical curve in $\R_{\max}^2$ is depicted in
Fig.~\ref{fig:tropical_curve_simple}. Every
half-ray corresponds to a different pair of maximizing terms: the
diagonal corresponds to $\{ (x, y): \ x = y > 0 \}$, the vertical half-ray to
$\{ (x, y): \ x = 0 > y \}$, and the horizontal to $\{ (x, y): \ y = 0 > x \}$.
More elaborate examples can be found in~\cite{MaSt15}.
\begin{figure*}[h!]
	\centering
	\begin{minipage}{0.48 \textwidth}
    \begin{tikzpicture}[->]%
    [square/.style={regular polygon,regular polygon sides=4},
     bluenode/.style={shape=circle, draw=blue, line width=2},
     magnode/.style={shape=circle, draw=magenta, line width=2},
     rednode/.style={shape=circle, draw=red, line width=2}]
        \draw[draw=hblue, thick, >=stealth] (0, 0) -- (1.41, 1.41);
        \draw[draw=lblue, thick, >=stealth] (0, 0) -- (0, -1.8);
        \draw[draw=lred, thick, >=stealth] (0, 0) -- (-1.8, 0);
		\node[fill=white] at (-2, 1) {$y > \max(0, x)$};
        \node[fill=white] at (2, -1) {$x > \max(0, y)$};
        \node[fill=white] at (-2, -1) {$0 > \max(x, y)$};
	\end{tikzpicture}
    \caption{Tropical curve of $p(x, y) = \max(x, y, 0)$}
    \label{fig:tropical_curve_simple}
	\end{minipage}~
	\begin{minipage}{0.48 \textwidth}
		\begin{tikzpicture}
		\begin{axis}[
		xmin=-30, xmax=15, ymin=-30, view={-45}{30}, ymax=25, colorbar,
		width=0.99 \linewidth, colormap/violet, xlabel=$x_1$,
		ylabel=$x_2$]
		\addplot3[surf,domain=-30:25] {
			max(x + y, 2 * x, x + 2 * y) + max(0, -y, 2 * x - 2 * y)
		};
		\end{axis}
		\end{tikzpicture}
		\caption{$g(x, y) = \max(x + y, 2 x, x + 2y) + \max(0, -y, 3x - 2y)$}
		\label{fig:maxpoly_sum}
	\end{minipage}
\end{figure*}
Informally, one can think of this duality as a one-to-one correspondence
between the vectors $\begin{pmatrix} b_i \\ \bm{c}_i \end{pmatrix}$ that
define the maximizing terms on each open sector, and open sectors of $V(p)$.
We will elaborate on this duality in~\cref{sec:tropical_geometry_connections}.

\section{Connections to Tropical Geometry}
\label{sec:tropical_geometry_connections}
With the definition of a tropical polynomial at hand, we can already draw some
connections between popular neural network models and tropical geometry.
We are concerned with the following cases:
\begin{itemize}
\item ReLU activations: given input $v = \bm{w}^\top \bm{x} + b$ with
$\bm{w}, \bm{x} \in \R^n$, a Rectifier Linear Unit computes
\begin{equation}
    \mbox{ReLU}(\bm{x}) = \max(0, \bm{w}^\top \bm{x} + b)
    \label{eq:relu_activation}
\end{equation}
\item
Maxout units: given $\bm{W} \in \R^{n \times k}$ and $\bm{b} \in
\R^k, \bm{x} \in \R^n$:
\begin{equation}
    \mbox{maxout}(\bm{x}) = \max_{j \in [k]} \left(
    \bm{W}_j^\top \bm{x} + b_j
    \right),
\end{equation}
where we denote $\bm{W}_j$ for the $j$-th row of $\bm{W}$.
\end{itemize}
A variation of ReLU for which this paper's results are
also applicable is the Leaky ReLU~\cite{MHN13}, which replaces the
standard activation function with
\begin{equation}
\mbox{LReLU}_{\alpha}(v) = \max(v, \alpha v), \quad 0 < \alpha < 1.
\label{eq:lrelu_activation}
\end{equation}
Notice that maxout units and ReLUs are tropical polynomials of rank $k$ and $2$,
respectively.
\subsection{Newton Polytopes of Tropical Polynomials}
Our investigation leverages a fundamental geometric object: the (extended)
\textbf{Newton Polytope} of a tropical polynomial. Given a polynomial as in
Eq.~\eqref{eq:tropical_signomial}, its corresponding Newton Polytope is defined as
in~\cref{eq:newton_polytope}.
\begin{align}
	\Nc(p) &:= \mathrm{conv}\left\{
	\left( \begin{array}{c} b_i \\ \bm{c}_i \end{array} \right) :
	i \in [m] \right\}
    \label{eq:newton_polytope}
\end{align}
Tropical addition and multiplication can also be interpreted as operations on
polytopes;~\cite{PaSt05} elaborate on applications of this interpretation.
\begin{prop} \label{prop:polytope_algebra}
    Let $h_1, \dots, h_m : \R_{\max}^n \to \R_{\max}$ be a collection of
    tropical polynomials. It holds that:
    \begin{align}
    	V\left( \sum_{i=1}^m h_i \right) &= \bigcup_{i=1}^m V(h_i)
    		\label{eq:tropplus_hypersurf} \\
        \Nc\left( \sum_{i=1}^m h_i \right) &=
            \Nc(h_1) \oplus \cdots \oplus \Nc(h_m)
            \label{eq:tropmul_newtp}
    \end{align}
\end{prop}
\begin{proof}
	The first identity can be found as Proposition 1.16 in~\cite{Hept09} for two
	polynomials and extended via induction. Importantly, its proof does not require
	the exponents to be integer-valed. For the second identity, consider
	\begin{align}
		h_1(\bm{x}) & := \bigvee_{i=1}^{k_1} \alpha_i + \bm{\beta}_i^\top
		\bm{x}, \;
		h_2(\bm{x}) := \bigvee_{i=1}^{k_2} \gamma_i + \bm{\delta}_i^\top \bm{x}
		\\
        (h_1 + h_2)(\bm{x}) &=
            \bigvee_{i \in [k_1], j \in [k_2]} \alpha_i + \gamma_j +
            (\bm{\beta}_i + \bm{\delta}_j)^\top \bm{x},
            \label{eq:sum-of-max}
	\end{align}
    where~\cref{eq:sum-of-max} follows from the identity
    $(a + b) \vee (c + d) = (a + c) \vee (b + c) \vee (a + d) \vee (b + d)$.
    However, the terms inside the maximum are precisely sums of individual
    terms of the two polynomials, so the claim follows. The proof can again be
    extended via induction.
\end{proof}

We present a few results about faces of polytopes that will be needed in Sec.~\ref{sec:linear_regions}. First, recall the definition for a special kind of polytope, called a \textit{zonotope}:
\begin{defn}
    A \textbf{zonotope} $Z \in \R^n$ is a polytope in $\R^n$ which can be
    written as the Minkowski sum of a set of line segments (edges).
    \label{defn:zonotope}
\end{defn}
The \textit{edgotope} is the \textit{smallest zonotope
covering $P$}:
\begin{defn} \label{defn:edgotope}
	The edgotope $Z(P)$ of a polytope $P$ is the Minkowski sum
	of all the \textit{edges} of $P$:
    \begin{equation}
        Z(P) := \bigoplus_{\bm{e} \in \mathrm{edges}(P)} \bm{e}
        \label{eq:edgotope}
    \end{equation}
\end{defn}
Proposition~\ref{prop:edgotope_polytope_inequality} is a remarkable inequality
between faces of polytopes and their edgotopes. Theorem~\ref{thrm:edgotope_polytope_faces_inequality} leverages it to upper
bound the faces of an arbitrary Minkowski sum. Both appear in~\citep[Section 2]{GriStu93}.
\begin{prop} \label{prop:edgotope_polytope_inequality}
    Let $f_i(P)$ denote the number of $i$-dimensional faces of a polytope $P$.
    Given polytopes $P_1, P_2, \dots, P_k \in \R^n$, we have:
    \begin{equation*}
        f_i(P_1 \oplus P_2 \dots \oplus P_k) \leq
        f_i\left( Z(P_1) \oplus Z(P_2) \dots \oplus Z(P_k) \right)
    \end{equation*}
\end{prop}
\begin{thrm} \label{thrm:edgotope_polytope_faces_inequality}
    Let $P_1, P_2, \dots P_k$ be polytopes in $\R^n$, $m$ denote the number
    of nonparallel edges of $P_1, P_2, \dots P_k$, and $i \in \set{0, \dots, n-1}$.
    Then
    \begin{equation}
        f_i(P_1 \oplus P_2 \dots \oplus P_k) \leq
        2 \binom{m}{i} \sum_{j=0}^{n - 1 - i} \binom{m - 1 - i}{j}
        \label{eq:edgotope_polytope_faces_inequality}
    \end{equation}
    Moreover, for $f_0(P_1 \oplus \dots \oplus P_k)$, which denotes the number
    of vertices of the Minkowski sum, the bound
    of~\eqref{eq:edgotope_polytope_faces_inequality} is tight when $2k >
    n$.
\end{thrm}
In~\cref{eq:edgotope_polytope_faces_inequality}, the right hand side
is the number of $i$-faces of a zonotope generated by $m$ line
segments.

\subsection{On the number of linear regions of ReLU/Maxout layers}
\label{sec:linear_regions}
Pioneering work on DNNs with piecewise-linear activation units focuses on
extracting bounds for the number of linear regions they
designate~\cite{PMB13,MPC+14}.
In our treatment, we extract asymptotically similar upper bounds for maxout
units and a tight upper bound for layers of rectifier networks, leveraging the
corresponding Newton polytopes. In~\cite{MPC+14}, the authors argue that the
number of linear regions of a maxout unit is upper bounded by its rank. In
fact, that number is in bijection with the number of vertices of the
\textit{upper hull} of the corresponding Newton polytope. The following appears
in~\cite{ChMa17} without proof:
\begin{prop}
    \label{prop:linear_regions_bias_vertices}
    Let $h(\bm{x})$, as in~\eqref{eq:tropical_signomial}, describe the
    activation of a maxout unit. Then there is a bijection between
    $h$'s linear regions and the vertices lying on the \textbf{upper hull} $\Nc^{\max}(h)$ of $\Nc(h)$.
\end{prop}
\begin{proof}
	Consider
    \begin{equation}
        \bm{c}' = \left( \begin{array}{c} b \\ \bm{c} \end{array} \right), \quad
        \bm{x}' = \left( \begin{array}{c} 1 \\ \bm{x} \end{array} \right).
        \label{eq:proof_extended_vectors}
    \end{equation}
    We can thus rewrite the maxpolynomial's response as a linear program:
    \begin{align}
        \begin{aligned}
            \mathrm{Maximize}\ & (\bm{x}')^\top \bm{c}' \\
            \mathrm{s.t.}\ & \bm{c}' \in \Nc(h)
        \end{aligned}
        \label{eq:extended_linear_program_proof}
    \end{align}
    From the fundamental theorem of linear programming~\cite{Vand14}, we know
    that optimal solutions to~\eqref{eq:extended_linear_program_proof}
    will lie at one of the vertices of $\Nc(h)$. However, the restriction
    of the first element of $\bm{x}'$ hints that some vertices might be
    redundant. Indeed, pick any vertex $\bm{c}_j' \notin \Nc^{\max}(h)$, which
    implies that $\exists \bm{c}_i' \in \Nc^{\max}(h)$, not necessarily
    a vertex, satisfying:
    \begin{align}
        (\bm{c}_j')_1 &= b_j \leq (\bm{c}_i')_1 = b_i, \quad
         \bm{c}_j = \bm{c}_i \label{eq:upper_hull_property} \\
         \Rightarrow
         {\bm{x}'}^\top \bm{c}_j' &=
         b_j + \bm{x}^\top \bm{c}_j         \overset{\eqref{eq:upper_hull_property}}{\leq}
         b_i + \bm{x}^\top \bm{c}_i = {\bm{x}'}^\top \bm{c}_i'.
         \label{eq:upper_hull_inequality}
	\end{align}
    Inequality~\eqref{eq:upper_hull_inequality} means that, if we let
    $\bm{c}'$ run over all of the Newton polytope, all points not in the upper
    hull are redundant. Every point in the upper hull that maximizes a linear
    functional either is a vertex, or can be substituted by a vertex in the
    upper hull that maximizes the same linear form, from which the claim follows.
\end{proof}

In~\cref{fig:newt_examples} we illustrate the canonical projections of the
Newton polytopes of the individual summands of $g(x, y)$, which is depicted
in~\cref{fig:maxpoly_sum}. It appears to designate a total of $4$ linear
regions, as Proposition~\ref{prop:linear_regions_bias_vertices} suggests.

\begin{figure*}
    \begin{minipage}{0.65 \textwidth}
    \begin{minipage}{0.15 \textwidth}
        \begin{center}
        \begin{tikzpicture}
        \filldraw[mblue, draw=black,opacity=0.6]
        (1, 1) -- (1, 2) -- (2, 0) -- (1, 1);
        \fill (1, 1) circle (2pt);
        \node[left=0.5pt of {(1, 1)}] {\half{$(1, 1)$}};
        \fill (1, 2) circle (2pt);
        \node[above=0.5pt of {(1, 2)}] {\half{$(1, 2)$}};
        \fill (2, 0) circle (2pt);
        \node[below=0.5pt of {(2, 0)}] {\half{$(2, 0)$}};
        \end{tikzpicture}
        \end{center}
    \end{minipage}
    \begin{minipage}{0.25 \textwidth}
        \begin{tikzpicture}
        \filldraw[mblue, draw=black,opacity=0.6]
        (0, 0) -- (0, -1) -- (2, -2) -- (0, 0);
        \fill (0, 0) circle (2pt);
        \node[above left=0.5pt of {(0, 0)}] {\half{$(0, 0)$}};
        \fill (0, -1) circle (2pt);
        \node[above=0.5pt of {(0, -1)}] {\half{$(0, -1)$}};
        \fill (2, -2) circle (2pt);
        \node[below=0.5pt of {(2, -2)}] {\half{$(2, -2)$}};
        \end{tikzpicture}
    \end{minipage}
    \begin{minipage}{0.24 \textwidth}
        \begin{center}
        \begin{tikzpicture}
        \filldraw[mblue, draw=black,opacity=0.6]
        (1, 0) -- (1, 2) -- (3, 0) -- (4, -1) -- (2, -1)-- (1, 0);
        \fill (1, 1) circle (2pt);
        \node[left=0.5pt of {(1, 1)}] {\scalebox{0.75}{$(1, 1)$}};
        \fill (1, 0) circle (2pt);
        \node[left=0.5pt of {(1, 0)}] {\scalebox{0.75}{$(1, 0)$}};
        \fill (1, 2) circle (2pt);
        \node[left=0.5pt of {(1, 2)}] {\scalebox{0.75}{$(1, 2)$}};
        \fill (3, 0) circle (2pt);
        \node[above right=0.5pt of {(3, 0)}] {\scalebox{0.75}{$(3, 0)$}};
        \fill (4, -1) circle (2pt);
        \node[below=0.5pt of {(4, -1)}] {\scalebox{0.75}{$(4, -1)$}};
        \fill (2, -1) circle (2pt);
        \node[below=0.5pt of {(2, -1)}] {\scalebox{0.75}{$(2, -1)$}};
        \fill (2, 0) circle (2pt);
        \node[above=0.5pt of {(2, 0)}] {\scalebox{0.75}{$(2, 0)$}};
        \fill (3, -1) circle (2pt);
        \node[below=0.5pt of {(3, -1)}] {\scalebox{0.75}{$(3, -1)$}};
        \end{tikzpicture}
        \end{center}
    \end{minipage}
    \caption{Projected Newton polytopes for the polynomial
    in~\cref{fig:maxpoly_sum}.
    Left and center: polytopes of the summands. Right: polytope of the sum.}
    \label{fig:newt_examples}
    \end{minipage} \quad
    \begin{minipage}{0.35 \textwidth}
        \centering
    	\begin{tikzpicture}[scale=1.85]

    	\coordinate (RA) at (0, 0);
    	\coordinate (RB) at (-1, 0);
    	\coordinate (RC) at (0, -1);
    	\coordinate (RD) at (0.707, 0.707);  

    	\draw[>=stealth, mred, thick] (RA) -- (RB);
    	\draw[>=stealth, mred, thick] (RA) -- (RC);
    	\draw[>=stealth, mred, thick] (RA) -- (RD);

    	\coordinate (sRA) at (0.5, -0.5);
    	\coordinate (sRB) at (0.5, -1.5);
    	\coordinate (sRC) at (-0.5, -0.5);
    	\coordinate (sRD) at (1.207, 0.207);

    	\draw[>=stealth, mblue, thick] (sRA) -- (sRB);
    	\draw[>=stealth, mblue, thick] (sRA) -- (sRC);
    	\draw[>=stealth, mblue, thick] (sRA) -- (sRD);

    		\node at (-0.25, 0.3) {{$\mathcal{R}_1$}};
    		\node at (0.25, -0.2) {{$\mathcal{R}_2$}};
    		\node at (-0.25, -0.2) {{$\mathcal{R}_3$}};
    		\node at (0.8, -0.6) {{$\mathcal{R}_4$}};
    		\node at (0.25, -0.75) {{$\mathcal{R}_5$}};
    		\node at (-0.25, -0.75) {{$\mathcal{R}_6$}};
    	\end{tikzpicture}
    	\caption{$V(p_1) \cup V(p_2)$ and corresponding linear regions}
    	\label{fig:union-hyp}
    \end{minipage}
\end{figure*}

\subsubsection{Upper bounds for Relu layers}
\cite{MPC+14} argue that a linear region in a ReLU layer corresponds to a
configuration of active units. Letting $\Nc^n_m$ denote the
number of linear regions of a ReLU layer with $n$ inputs and $m$ outputs, this
observation immediately gives $\Nc^n_m \leq 2^m$. Using the
notion of the Newton polytope, we can derive tighter bounds:
\begin{prop} \label{prop:relu_zonotope}
    Let $h_i(\bm{w}_i, b_i) = \max(0, \bm{w}_i^\top \bm{x} + b_i),\; i
    = 1, \dots m$ be an arbitrary collection of rectifier units. Then,
    the Minkowski sum $h_1 \oplus \dots \oplus h_m$ has at most
    $k$ nonparallel edges.
\end{prop}
\begin{proof}
    By definition, $\Nc(h_i)$ is a zonotope since $h_i$ is a rank-$2$ polynomial.
    Zonotopes are line segments, so the Minkowski sum of $k$ such zonotopes has
    at most $k$ nonparallel edges.
\end{proof}
Notice that Proposition~\ref{prop:relu_zonotope} still holds for leaky ReLUs,
in which case \[\Nc(h_i) = \convhull\set{
                \begin{pmatrix}
                    \alpha b \\ \alpha \bm{w}
                \end{pmatrix},
                \begin{pmatrix}
                    b \\ \bm{w}
                \end{pmatrix}}. \]

Assume we are given a collection of ReLUs (i.e. a layer). Each of these ReLUs
is a polynomial $p_i: \Rbb^n \to \Rbb$, therefore the total number of linear
regions is dual to the hypersurface of that collection of polynomials, which is
$V(p_1) \cup \dots V(p_m)$ (see~\cref{fig:union-hyp}).
By~\cref{eq:tropplus_hypersurf}, this is the same as $V\left(\sum_{i=1}^m
p_i\right)$, which by~\cref{eq:tropmul_newtp} is dual to $\Nc(p_1) \oplus \cdots
\oplus \Nc(p_m)$. The latter is itself a Newton polytope of a polynomial, hence
only vertices on its upper hull correspond to linear regions of the collection
$\set{p_i}_{i=1}^m$. Proposition~\ref{prop:linear_regions_bias_vertices}
specializes that fact to a single polynomial.

Theorem~\ref{thrm:edgotope_polytope_faces_inequality} together with
Prop.~\ref{prop:relu_zonotope} then suggest that:
\begin{equation}
  f_i(\Nc(h_1) \oplus \cdots \oplus \Nc(h_k))
  = 2 \binom{k}{i} \sum_{j=0}^{n - i} \binom{k - 1 - i}{j}
    \label{eq:relu_faces}
\end{equation}
Moreover, it is known that zonotopes are centrally symmetric (see
e.g.~\cite{BeckRob15}), which implies that their upper and lower hulls have the
same number of vertices. Consequently:
\begin{prop} \label{prop:relu_tight_bound}
	The number of linear regions of a ReLU/LReLU layer with $n$ inputs and
    $m$ outputs is upper bounded as
    \begin{equation}
    	\Nc_m^n \leq \min \left(
        	2^m, \sum_{j=0}^n \binom{m}{j}
        \right)
    \end{equation}
    Moreover, this bound is tight when the zonotopes
    corresponding to the ReLU activations, as well as the canonical projection
    to the last $n$ coordinates of its vertices, are in general position.
\end{prop}
\begin{proof}
By the preceding discussion, it is clear than a ReLU layer with $m$ outputs
defines a union of $m$ hypersurfaces, $\setU_{i=1}^m V(h_i)$.
By Prop.~\ref{prop:polytope_algebra}, this is equal to $V(\sum_{i=1}^m h_i)$.
Therefore, it suffices to upper bound the number of vertices on the upper hull
of
\[
    \Nc\left(\sum_{i=1}^m h_i\right) =
    \Nc(h_1) \oplus \dots \oplus \Nc(h_m).
    \qquad (\text{By}~\eqref{eq:tropmul_newtp})
\]
From then, the proof is an application of
Theorem~\ref{thrm:edgotope_polytope_faces_inequality},
Prop.~\ref{prop:relu_zonotope} and
Prop.~\ref{prop:edgotope_polytope_inequality}, in which the inequality is tight
since $Z(P_i) = P_i$ for any zonotope $P_i$. Notice that a zonotope $P$ being
centrally symmetric means that its lower and upper hulls have the same number
of vertices, say $n_{\ell} = n_u = n$. However, its total number of vertices
$\abs{\vtx(P)} \neq 2n$ in general, since it's possible to have vertices in
both the lower and upper hulls at the same time, as~\cref{fig:lower-upper-hull}
shows. Another example of such a zonotope is the $\ell_1$-ball in $d \geq 2$
dimensions.

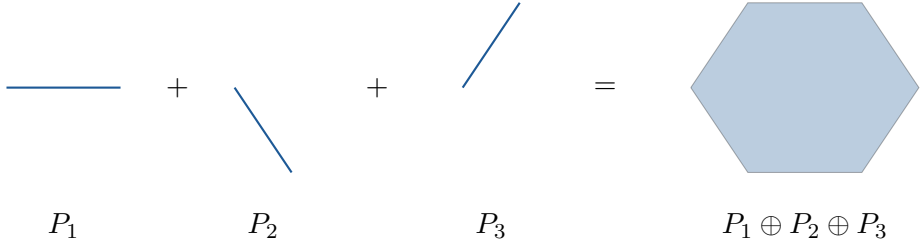
\begin{figure}[h!]
    \centering
 	\begin{tikzpicture}[scale=1.5]
		\draw[draw=hblue, thick, -] (0, 0) -- (1, 0);
		\node at (1.5, 0) {$+$};
		\draw[draw=hblue, thick, -] (2, 0) -- (2.5, -0.75);
		\node at (3.25, 0) {$+$};
		\draw[draw=hblue, thick, -] (4, 0) -- (4.5, 0.75);
		\node at (5.25, 0) {$=$};
        \filldraw[hblue, draw=black, opacity=0.3]
			(6, 0) -- (6.5, 0.75) -- (7.5, 0.75) -- (8, 0) -- (7.5, -0.75) --
				(6.5, -0.75) -- (6, 0);
		\node[below=1pt of {(0.5, -1)}] {$P_1$};
		\node[below=1pt of {(2.25, -1)}] {$P_2$};
		\node[below=1pt of {(4.25, -1)}] {$P_3$};
		\node[below=1pt of {(7, -1)}] {$P_1 \oplus P_2 \oplus P_3$};
 	\end{tikzpicture}
 	\caption{Zonotope with vertices in both envelopes.}
 	\label{fig:lower-upper-hull}
\end{figure}
Denote $P^{\max}, P^{\min}$ for the upper and lower hulls respectively.
A vertex $\bm{v} \in P^{\max} \cap P^{\min}$ if it is also a vertex for the
canonical projection of $P \in \Rbb^n$ to the last $ n - 1 $ coordinates,
denoted by $P'$.
Therefore:
\begin{align}
	\abs{\vtx(P)} &= \abs{\vtx(P^{\max})} + \abs{\vtx(P^{\min})} - \abs{
		\vtx(P')} \\
		&= 2n - \abs{\vtx(P')} \Rightarrow
		   n = \frac{\abs{\vtx(P)} + \abs{\vtx(P')}}{2}.
		   \label{eq:upper-hull-sum}
\end{align}
Theorem~\ref{thrm:edgotope_polytope_faces_inequality} applied for $P$ and $P'$
tells us that the right hand side in~\cref{eq:upper-hull-sum} is bounded above
by
\begin{align}
	\sum_{j=0}^n \binom{m - 1}{j} &+ \sum_{j=0}^{n - 1} \binom{m - 1}{j}
	\label{eq:two-zonotope-upper-bounds}
	= 1 + \sum_{j=1}^n \binom{m - 1}{j} + \binom{m - 1}{j - 1} \\
	&= 1 + \sum_{j=1}^n \binom{m}{j} = \sum_{j=0}^n \binom{m}{j},
\end{align}
where we've made use of the identity $\binom{n}{k} = \binom{n-1}{k} +
\binom{n-1}{k-1}$. This gives one part of the claimed bound. The other part
of the claimed bound follows from the argument in~\cite{MPC+14}, i.e. the number
of possible ReLU patterns is bounded above by $2^m$. The claim follows.
\end{proof}

The result above assumes a fully-connected neural network layer. It is
straightforward to obtain a similar bound for convolutional layers.
For a convolutional layer, one may write $\bm{y} = \bm{W}
\mathrm{vec}(\bm{X})$, where $\mathrm{vec}(\cdot)$ ``reshapes'' its argument
into a single vector, and deduce the following:
\begin{corl} \label{corl:conv-layer}
	The number of linear regions of a single-channel ReLU/LReLU convolutional
	layer with filter size $k$ and padding $p$, applied on square images of
	size $d^2$, is upper bounded by
	\[
	\min \left(
	2^{(d - k + 2p + 1)^2},
	\sum_{j=0}^{d^2} \binom{(d - k + 2p + 1)^2}{j}
	\right).
	\]
\end{corl}
\begin{proof}
A convolutional layer applies a 2D convolution to the set of input images
\[
	\set{\bm{X}_i}_{i=1}^n, \; \bm{X}_i \in \Rbb^{d_w \times d_h}, \;
\]
where $d_w, d_h$ are the width and height of the images (assume
single-channel). Equivalently, $m$ filters of size $k \times k$ are applied to
$\bm{X}_i$ on (possibly) overlapping regions. We now assume that those regions
are separated by a stride of size $1$, but our analysis extends in a
straightforward way to the case where we have larger strides. In practice,
images are also zero-padded by $p$ pixels.

When the conv-layer's activations are ReLUs or leaky ReLUs, our previous
arguments apply in a straightforward fashion. The dimension of the output is
$d_{\mathrm{out}} = (d_w + 2p - k + 1) \times (d_h + 2p - k + 1)$.
The convolution operation is an affine mapping $\bm{X} \mapsto \bm{W}
\mathrm{vec}(\bm{X}) + \bm{b}$, where $\mathrm{vec}(\bm{X})$
denotes the vectorization of $\bm{X}$. The weight matrix has at least $1$ and
at most $k^2$ elements on every row. By our previous arguments, this will
result in a collection of $d_{\text{out}}$ tropical signomials. The case of
interest is square images with $d_w = d_h = d$, which results in
$d_{\mathrm{in}} = d^2, \; d_{\mathrm{out}} = (d - k + 2p + 1)^2$. Then, an
application of Prop.~\ref{prop:relu_tight_bound} gives the result.
\end{proof}

\subsubsection{Upper bounds for Maxout layers}
By a similar argument, we can recover bounds for maxout units. Let
$h(\bm{x})$ be a maxout activation of rank $k$, which defined at most $k$
linear regions; by our observation its Newton polytope will have at most $k$
vertices. Therefore, the maximal number of edges it will contain is
\(
\binom{k}{2} = \frac{k(k-1)}{2}
\).
If we also assume that all the edges of all $m$ polytopes are
in general position, we immediately arrive at
\begin{corl}
	The linear regions of a maxout layer of $n$ inputs and $m$ outputs,
    using units of rank $k$, are upper bounded by
    \begin{equation}
    	\min \left(
        	k^m,
            2 \cdot \sum_{j=0}^n \binom{m \cdot \frac{k(k-1)}{2}}{j}
        \right)
    \end{equation}
    \label{corl:maxout_tight_bound}
\end{corl}
The same bound holds for the linear regions of
\[
	g_{+}(\bm{x}) = \sum_{i=1}^m w_i \cdot h_i(\bm{x}), \quad \bm{w} \geq 0,
\]
when $\set{h_i}_{i=1}^m$ are rank-$k$ tropical polynomials, since $\Nc(g_+)$ is
the Minkowski sum of scaled Newton polytopes of $h_i$. Notice that we cannot
refine the binomial sum in Corollary~\ref{corl:maxout_tight_bound}, as the
resulting Newton polytope is not necessarily centrally symmetric.

\section{Counting linear regions in practice}
In this section, we provide a computational method to measure the expressive
power of a neural network layer, by enumerating its linear regions. In contrast
to approaches relying on mixed-integer programming (MIP) such
as~\cite{SerTjaRam18,SerRam18}, which usually
assume that the input data are bounded in some range, we make no such
assumption here.

Suppose we are given $m$ piecewise-linear activation functions
$\set{h_i}_{i=1}^m$ such that
$
    h_i = \bigvee_{j=1}^{k_i} \bm{W}_{i,j}^\top \bm{x} + b_{i,j}.
$
Knowing $h_i$ immediately gives us a (not necessarily minimal)
$\cV$-representation of the corresponding polytope $P_i =
\convhull\set{\begin{pmatrix} \bm{W}_{i,1} \\ b_{i,1} \end{pmatrix}, \dots,
\begin{pmatrix} \bm{W}_{i, k_i} \\ b_{i, k_i}
\end{pmatrix}}$. It thus suffices to compute the number of vertices
in the upper hull of the Minkowski sum $P_1 \oplus \dots \oplus P_m$.

\paragraph{Exact counting for a single layer.}
It is widely known that the extreme points of Minkowski sums of polytopes are
sums of extreme points of the individual polytopes. Additionally, there exist
algorithms for enumerating vertices of Minkowski sums of polytopes $P_1, \dots
P_m$ when the $\cV$-representation of the $P_i$'s is available: this has become
widely known as the \textit{reverse search} method~\cite{AviFuk96,Fukuda04}.

Theorem 3.3 in~\cite{Fukuda04} proves the existence of a polynomial
algorithm for enumerating the vertices of $P := P_1 \oplus \dots \oplus P_m$
in time $\cO(\sum_i \delta_i \mathrm{LP}(n, \delta)
\abs{\mathrm{vert}(P)})$, where is $\delta_i$ the maximum degree of the vertex
adjacency graph of $P_i$ and
$\mathrm{LP}(n, \delta)$ denotes the time required to solve a linear program
(LP) in $n$ variables and $\delta$ inequalities. Combined with our estimates, that implies straightforward bounds for exact counting of the linear regions of ReLU/LReLU/Maxout layers. In our case, $\delta = 2m$ for ReLU/LReLU layers and $\delta = \sum_i k_i$ in the case of general convex PWL functions.

Let us briefly address the issue of having a non-minimal $\cV$ representation
for some of the polytopes $P_i$. In the case of a ReLU/LReLU network, all
polytopes $P_i$ will be edgotopes, which will admit a minimal $\cV$
representation unless $\bm{W}_{i} = 0$. In the case of a Maxout network,
we can eliminate redundant terms by solving $k_i$ LPs (see~\cite{ParLiHag95}
for more details).

Unfortunately, counting the vertices using reverse search requires solving
a prohibitive number of LPs, rendering the approach outlined
above impractical. Recent approaches count linear regions using mixed-integer
formulations that effectively identify the activation patterns
of rectifier networks (e.g.~\cite{SerRam18}). We attack this problem
from a different angle, by considering the ``dual'' problem of counting
vertices of convex polytopes by sampling.

\begin{figure*}[htb!]
	\centering
	\begin{minipage}{0.4 \linewidth}
		\centering
		\begin{tikzpicture}[thick,scale=1]
		\coordinate (A1) at (0, 0);
		\coordinate (A2) at (2, 1);
		\coordinate (A3) at (1, 3);
		\coordinate (A4) at (-1.5, 2);
		\fill[hblue, opacity=0.3] (A1) -- (A2) -- (A3) -- (A4) -- cycle;
		\draw[dashed, opacity=0.6] (A1) -- (A2) -- (A3) -- (A4) -- cycle;
		\draw (0.2, 1.5) node {$P$};
		\end{tikzpicture}
		\caption{Regular solid angles}
		\label{fig:angle-good}
	\end{minipage}
	\qquad
	\begin{minipage}{0.4 \linewidth}
		\centering
		\begin{tikzpicture}[thick,scale=1]
		\coordinate (A1) at (0, 0);
		\coordinate (Aext) at (1, 0.3);
		\coordinate (A2) at (2, 1);
		\coordinate (A3) at (1, 3);
		\coordinate (A4) at (-1.5, 2);
		\coordinate (B1) at (1.3, -0.7);
		\coordinate (B2) at (1.6, -0.5571);
		\fill[darkorange, opacity=0.3] (Aext) -- (B1) -- (B2) -- cycle;
		\draw[dotted] (Aext) -- (B1);
		\draw[dotted] (Aext) -- (B2);
		\fill[hred, opacity=0.3] (A1) -- (Aext) -- (A2) -- (A3) --
		(A4) -- cycle;
		\draw[dashed, opacity=0.6] (A1) -- (Aext) -- (A2) -- (A3) -- (A4) --
		cycle;
		\draw [dotted, opacity=0.6] (A1) -- (A2);
		\draw (0.2, 1.5) node {$Q$};
		\draw (1, 0.25) node[below left] {$\bm{v}_i$};
		\draw (1.6, -0.5571) node[right] {$\color{darkorange} N_Q(\bm{v}_i)$};
		\end{tikzpicture}
		\caption{$\omega(N_Q(\bm{v}_i)) \ll 1$}
		\label{fig:angle-bad}
	\end{minipage}
\end{figure*}

\subsection{A sampling method for polytopes}
We briefly present a randomized heuristic for ``sampling'' the extreme points
of the upper hull of a polytope $P = P_1 \oplus \dots \oplus P_m$. We generate
$K$ standard normal vectors, i.e. $\bm{g}^k \sim_{\mathrm{i.i.d}}\normal(\bm{0},
\bm{I})$ and compute $\ip{\bm{g}^k, \bm{v}_i}, \; \forall$ extreme point $\bm{v}_i$. We record the minimizers/maximizers for each polytope $P_j$ and repeat the trial.
This gives us a lower bound for the total number of vertices in the Minkowski
sum, since it is well-known that extreme points of a polytope are maximizers of
linear functionals over it, and extreme points of Minkowski sums maximize the
same linear functional over all individual summands. Let
\[
	\bm{V}_i = \begin{pmatrix}
		{\bm{v}^i_1} &
		\dots &
		{\bm{v}^i_{k_i}}
	\end{pmatrix}^\top \in \Rbb^{k_i \times n}, \; \forall i \in [m],
\]
each row of which is a vertex of $P_i$. By convention, the first coordinate of
each row contains the bias term. Our proposed
method,~\cref{alg:extreme-point-counting}, leverages the techniques
in~\cite{DamSun17}. We stress that this method and its specialization to upper
hulls,~\cref{alg:extreme-point-counting-upper-hull},
work for \textit{general}
polytopes, while the mixed-integer-program based methods in the
literature are only presented for rectifier networks.

\begin{algorithm}
    \begin{algorithmic}
    \State {\bfseries Input:} polytopes $P_1, \dots, P_m$ in
    $\cV$-representation
    \State $I_{\mathrm{ext}} := \emptyset$.
	\For{$j = 1, \dots, K$}
		\State Sample $\bm{g}_j \sim \normal(\bm{0}, \bm{I}_n)$
		\State Compute $\bm{z}^i := \bm{V}_i \bm{g}_j, \; \forall i \in [m]$.
		\State Collect $\bm{z}_{\max} := (\argmax \bm{z}^1, \dots, \argmax
		\bm{z}^m)$, $\bm{z}_{\min} := (\argmin \bm{z}^1, \dots, \argmin \bm{z}^m)$.
		\State $I_{\mathrm{ext}} := I_{\mathrm{ext}} \cup
		\set{ \bm{z}_{\max}, \bm{z}_{\min} }$
	\EndFor
    \end{algorithmic}
	\caption{Sampling points in the convex hull}
	\label{alg:extreme-point-counting}
\end{algorithm}
\cref{alg:extreme-point-counting} provides a nontrivial lower bound to the
number of extreme points of the resulting Minkowski sum with high probability,
as Proposition~\ref{prop:hp-estimates} shows.

\begin{prop} \label{prop:hp-estimates}
	Let $N = \abs{\mathrm{vert}\left(P_1 \oplus \dots \oplus P_m\right)}$
    and denote
    \[
        \tilde{N} = \log\left( \frac{1}{
        \max_k \left(1 - 2 \omega(N_P(\bm{v}_k))\right)} \right)
        \geq \frac{N}{2}.
    \]
	Then, for $K \geq \tilde{N} \log(N / \delta)$
    in~\cref{alg:extreme-point-counting}, the algorithm counts all the
    vertices with probability at least $1 - \delta$.
\end{prop}
\begin{proof}
    An extreme point of a Minkowski sum is necessarily a sum of extreme points
    of individual summands. Each time we draw a random sample
    $\bm{g}_j$ and record the minimizers of $\set{\bm{V}_i \bm{g}_j}_{i \in
    [m]}$, we are recording one possible extreme point of $P_1 \oplus \dots
    \oplus  P_m$. Consequently, missing a ``configuration'' of minimizers
    across our trials is equivalent to missing an extreme point $\bm{v}$ of the
    Minkowski Sum.

    Enumerate the individual vertices as $\bm{v}_1, \dots,
    \bm{v}_{N}$. Then,
    \begin{align}
        \prob{\text{fail}} &=
        \prob{\setU_{k=1}^{N} \text{miss } \bm{v}_k}
        \overset{(\text{union bound})}{\leq}
        \sum_{k=1}^{N} \prob{\text{miss } \bm{v}_k}
        \label{eq:miss-union-bound}
    \end{align}
    ``Missing'' $\bm{v}_k$ means that it was not a minimizer for any functional
    $\ip{\bm{g}_j, \cdot}$; equivalently (by independence across samples):
    \begin{align}
        \prob{\text{miss } \bm{v}_k} = \prob{\setI_{j = 1}^K \set{
        \pm \bm{g}_j \notin N_P(\bm{v}_k)}}
        \nonumber \\
        = \prod_{j=1}^K \left[ 1 - \prob{\pm \bm{g}_j \in N_P(\bm{v}_k)} \right]
        \leq
        \left(1 - 2 \omega(N_P(\bm{v}_k))\right)^K \\
        \Rightarrow \prob{\text{miss a vertex}} \leq N
        \max_k \left(1 - 2 \omega(N_P(\bm{v}_k))\right)^K
    \end{align}
    If we require the above to be less than $\delta$, we obtain
    $ \delta \geq N \max_k \left(1 - 2 \omega(N_P(\bm{v}_k))\right)^K$,
    which gives the result.
\end{proof}
\vspace*{-2pt}
Our guarantee heavily depends on the cones $N_P(\bm{v}_k)$. If there are
vertices that only slightly ``extend'' out of the polytope, our required
sample size will be a large multiple of
$N$.~\Cref{fig:angle-good,fig:angle-bad} illustrate (non-zonotopal)
examples in $\Rbb^2$; $Q$ has a vertex where the solid angle of the normal cone
is close to $0$, in contrast to $P$ which is more ``regular''.
If one can ``get away'' with computing a lower bound on the actual number
of linear regions, a similar guarantee is available; instead of
the exact number of linear regions we may consider a threshold $\frac{1}{2} >
\eta > 0$ and the set $\cV_{\eta} := \set{\bm{v}_i \in \mathrm{vert}(P) \mmid
\omega(N_P(\bm{v}_i)) \geq \eta}$; informally, $\cV_{\eta}$ is the set of
vertices whose normal cones' angles are not ``too small'.
\begin{corl} \label{corl:random-lower-bound}
	Let $\eta$ be such that $\abs{\cV_{\eta}} \geq c N$, for some $c \in [0,
	1]$.
	Then Algorithm~\ref{alg:extreme-point-counting} counts at least $c N$
	vertices with probability at least $1 - \delta$, for $K \geq
	\frac{1}{2 \eta} \log \frac{N}{\delta}$.
\end{corl}
\begin{proof}
	We follow the proof of Prop.~\ref{prop:hp-estimates}, making use of the
	inequality $1 - x \leq e^{-x}$ to simplify the expression:
	\begin{align*}
	\prob{\text{miss from $\cV_{\eta}$}} =
	\prob{\setU_{\bm{v} \in \cV_{\eta}} \set{\text{miss $\bm{v}$}}} \\
	\leq
	\sum_{\bm{v} \in \cV_{\eta}} \prob{\text{miss $\bm{v}$}}
	\leq \abs{\cV_{\eta}} \max_{\bm{v} \in \cV_{\eta}} \left(1 - 2
	\omega(N_P(\bm{v}))\right)^K \\
	\leq N  \exp\left(-K
	\min_{\bm{v} \in \cV_{\eta}} 2 \omega(N_P(\bm{v}))\right)
	\leq N \exp(-2K\eta)
	\end{align*}
	Setting $N \exp(-2 K \eta) \leq \delta$ gives us
	$ K \geq \frac{1}{2 \eta} \log \frac{N}{\delta}$.
\end{proof}
Unfortunately, the correct parameter $\eta$ in
Corollary~\ref{corl:random-lower-bound} is not known a priori. Bounding the
(expected) number of vertices of the Minkowski sum when the generating
distribution of vertices of the summands is known (e.g. using some empirical
initialization rule, such as in~\cite{HeZhaRen+15}), is deferred to future work.

\paragraph{What about the upper hull?}
The analysis of~\cref{alg:extreme-point-counting} assumed that we are counting \textit{all} vertices of $P$; however, in our setting, we are only interested
in the upper hull. It is known that $\bm{v} \in P^{\min}$ implies that $c \in N_{P}(\bm{v}) \Rightarrow c_1 \leq 0$, so it suffices to consider only samples $\bm{g}_j$ with $(\bm{g}_j)_1 > 0$. We thus obtain a similar guarantee, stated
in Corollary~\ref{corl:upper-hull-count}.
\begin{algorithm}
	\begin{algorithmic}[1]
		\State {\bfseries Input:} polytopes $P_1, \dots, P_m$ in
		$\cV$-representation
		\State $I_{\mathrm{ext}} := \emptyset$.
		\For{$j = 1, \dots, K$}
		\State Sample $\bm{g}_j \sim \normal(\bm{0}, \bm{I}_n)$
		\If {$(\bm{g}_j)_1 < 0$}
			\State $\bm{g}_j := -\bm{g}_j$
		\EndIf
		\State Compute $\bm{z}^i := \bm{V}_i \bm{g}_j, \; \forall i \in [m]$.
		\State $\bm{z}_{\max} := (\argmax \bm{z}^1, \dots, \argmax \bm{z}^m)$
		\State $I_{\mathrm{ext}} := I_{\mathrm{ext}} \cup \set{ \bm{z}_{\max}
		}$
		\EndFor
	\end{algorithmic}
	\caption{Sampling points in the upper hull}
	\label{alg:extreme-point-counting-upper-hull}
\end{algorithm}
\begin{corl} \label{corl:upper-hull-count}
	Let $N$ denote the number of vertices on the upper hull of $P :=
	P_1 \oplus \dots \oplus P_m$, $\set{\bm{v}_k}_k$ be an enumeration of the vertices in $P^{\max}$, and $N'_P(\bm{v}) := \set{\bm{c} \in N_P(\bm{v}) \mmid c_1 \geq 0}$. Set
    \(
        \tilde{N} = \log\left( \frac{1}{
        \max_k \left(1 - \omega(N'_P(\bm{v}_k))\right)} \right).
    \)
	Then, for $K \geq \tilde{N} \log(N / \delta)$, Algorithm~\ref{alg:extreme-point-counting-upper-hull} counts all the
    vertices in $P^{\max}$ with probability at least $1 - \delta$.
\end{corl}
\begin{proof}
We follow the proof of Proposition~\ref{prop:hp-estimates}, with the slight
alteration that the number of extreme points calculated at each step is just
one. Enumerate the individual vertices as $\bm{v}_1, \dots, \bm{v}_{N}$. Again,
the union bound gives us
\begin{align}
    \prob{\text{fail}} \leq \sum_{k=1}^{N} \prob{\text{miss } \bm{v}_k}
    \label{eq:miss-union}
\end{align}
Now, consider a functional $\ip{\bm{g}_j, \cdot}$. Let us define
\begin{equation}
    \bm{q}_j := \begin{cases}
        \bm{g}_j, & \text{ if } (\bm{g}_j)_1 < 0 \\
        -\bm{g}_j, & \text{ otherwise}.
    \end{cases}
    \label{eq:positive-samples}
\end{equation}
Notice that setting $\bm{q}_j := -\bm{g}_j$ does not change the underlying
distribution $\normal(\bm{0}, \bm{I}_n)$, since centered normal random
variables are symmetric. Again, ``missing'' $\bm{v}_k$ and its interpretation
in terms of the truncated normal cones $N'_P$ means
\begin{align}
    \prob{\text{miss } \bm{v}_k} = \prob{\setI_{j = 1}^K \set{
    \bm{g}_j \notin N'_P(\bm{v}_k)}}
    \nonumber \\
    = \prod_{j=1}^K \left[ 1 - \prob{\bm{g}_j \in N'_P(\bm{v}_k)} \right]
    \leq \left(1 - \omega(N'_P(\bm{v}_k))\right)^K \\
    \Rightarrow \prob{\text{fail}} \leq N
    \max_k \left(1 - \omega(N'_P(\bm{v}_k))\right)^K
\end{align}
Notice that since we are only considering vertices in the upper hull of $P$,
it must hold that $N'_P(\bm{v}_k) > 0$, so the bound above is indeed not
vacuous. Requiring $\prob{\text{fail}} < \delta$ gives us the claimed lower
bound for $K$.
\end{proof}

\section{Conclusion}
We presented a unifying approach to bounding the number of linear regions of
neural networks using maxout/ReLU activations by treating the latter as
polynomials in tropical algebra. We showed that linear regions are in bijection
with vertices of the Newton polytopes of corresponding tropical polynomials,
which we leveraged to recover upper bounds.
Finally, we introduced a sampling algorithm for approximately counting the
number of linear regions of a single piecewise-linear layer. Our algorithm does
not impose any assumptions over the range of the input, avoids the
computational overhead of LP/MIP-based approaches, and extends beyond rectifier
networks. We hope that this contribution serves as a further step towards
underlining the importance of algebraic geometric methods in understanding the
complexity of learning models such as deep neural networks.

\bibliographystyle{plain}
\bibliography{references}

\clearpage

\end{document}